\documentclass[twoside,11pt]{article}

\usepackage{jmlr2e}

\setcitestyle{authoryear,open={(},close={)}}
\usepackage{amssymb}
\usepackage{pifont}
\newcommand{\xmark}{\ding{55}}%
\usepackage{microtype}
\usepackage{graphicx}
\usepackage{subfigure}
\usepackage{booktabs} 
\usepackage{enumerate}
\usepackage{amsmath}
\usepackage{amssymb}
\usepackage{amsfonts}
\usepackage{float}
\allowdisplaybreaks
\usepackage{thmtools, thm-restate}

\allowdisplaybreaks
\usepackage{thmtools, thm-restate}
\usepackage[utf8]{inputenc}
\usepackage[english]{babel}
\usepackage[inline]{enumitem}
\usepackage{ifthen}
\newboolean{showNotes}
\setboolean{showNotes}{True}
\newcommand\Naji[1]{\ifthenelse{\boolean{showNotes}}{\textbf{\textcolor{blue}{Naji: #1}}}{}}
\newcommand\Kun[1]{\ifthenelse{\boolean{showNotes}}{\textbf{\textcolor{blue}{Naji: #1}}}{}}
\newcommand\Peter[1]{\ifthenelse{\boolean{showNotes}}{\textbf{\textcolor{blue}{Naji: #1}}}{}}

\usepackage{scalerel,stackengine}
\stackMath
\newcommand\reallywidehat[1]{%
\savestack{\tmpbox}{\stretchto{%
  \scaleto{%
    \scalerel*[\widthof{\ensuremath{#1}}]{\kern-.6pt\bigwedge\kern-.6pt}%
    {\rule[-\textheight/2]{1ex}{\textheight}}
  }{\textheight}%
}{0.5ex}}%
\stackon[1pt]{#1}{\tmpbox}%
}
\usepackage{hyperref}
\usepackage{cancel}

\usepackage[normalem]{ulem}

\usepackage{bm}

\newtheorem{assumption}{Assumption}

\usepackage[linesnumbered,ruled,vlined]{algorithm2e}
\SetKwRepeat{Do}{do}{while}

\ShortHeadings{Learning from Positive and Unlabeled Data by Identifying The Annotation Process}{Shajarisales and Spirtes and Zhang}
\firstpageno{1}

\begin{document}

\title{Learning from Positive and Unlabeled Data by Identifying The Annotation Process}

\author{\name Naji Shajarisales \email najis@cmu.edu \\
       \addr Carnegie Mellon University\\
       Pittsburgh, PA 15213, USA
       \AND
\name Peter Spirtes \email ps7z@andrew.cmu.edu \\
       \addr Carnegie Mellon University\\
       Pittsburgh, PA 15213, USA
       \AND
\name Kun Zhang \email kunz1@cmu.edu \\
       \addr Carnegie Mellon University\\
       Pittsburgh, PA 15213, USA
}
\maketitle

\begin{abstract}
In binary classification, Learning from Positive and Unlabeled data (LePU) is semi-supervised learning but with labeled elements from only one class. Most of the research on LePU relies on some form of independence between the selection process of annotated examples and the features of the annotated class, known as the Selected Completely At Random (SCAR) assumption.  Yet the annotation process is an important part of the data collection, and in many cases it naturally depends on certain features of the data (e.g., the intensity of an image and the size of the object to be detected in the image). Without any constraints on the model for the annotation process, classification results in  the LePU problem will be highly non-unique. So proper, flexible constraints are needed. In this work we incorporate more flexible and realistic models for the annotation process than SCAR, and more importantly, offer a solution for the challenging LePU problem. On the theory side, we establish the identifiability of the properties of the annotation process and the classification function, in light of the considered constraints on the data-generating process. We also propose an inference algorithm to learn the parameters of the model, with successful experimental results on both simulated and real data.  We also propose a novel real-world dataset for LePU, 
as a benchmark dataset for future studies.
\end{abstract}

\section{Introduction}
\label{sec:intro}
Is an intelligent agent able to learn a concept of correct and
and incorrect when only exposed to correct instances? The answer
is surprisingly yes, at least when it comes to learning a
language. Known as “the paradox of the language acquisition”
\citep{jackendoff1997architecture}, human infants learn a language
almost solely based on being exposed to correct sentences
and words, and at some point in their development they
acquire their mother-tongue and speak it near-flawlessly.

Motivated by such a problem, we are interested in the following
learning problem: Suppose we want to learn a binary classification function. Our training set consists of a collection of unlabeled examples/data-points and a collection of labeled examples, but only from one class. More precisely if we represent the set of all possible instances with $\mathcal{X}$ and the binary classes set with $\mathcal{Y}=\{0, 1\}$, the training data is a set of size $N+M$ where $\{X_i, y_i\}_{i=1}^N\in \mathcal{X}\times\mathcal{Y}$ such that $y_i=1$, and another set $\{X_j\}_{j=N+1}^{N+M}$ where the class of the examples are not available. Can we still leverage the learning process, that under certain conditions, our learner encounters little or no loss as both $N$ and $M$ increase? This problem in the literature is known as the problem of Learning from Positive and Unlabeled data (LePU), or  different abbreviations PU-learning or LPU \citep{li2005learning}.

Before discussing a few real-world problems when one naturally faces LePU, we would like to make a remark.  To avoid confusion for any example $X_i$, we call $y_i$ its class, and \textit{not its label}. Instead here by labeled (or annotated) we mean a particular example $X_i$ is chosen and annotated by an expert, i.e., the expert has determined the class $X_i$ belongs to, namely $y_i$. In fact in machine learning literature ``label'' and ``class'' are used synonymously. However we will only use class to refer to $y_i$.  We make this distinction, so that following the methodology of \citep{elkan2008learning} we could introduce a new random variable $l_i$ that indicates weather $X_i$ is labeled/annotated by an expert, or not. We assign $l_i=1$ when  $X_i$ is among the annotated examples by the expert. And $l_i=0$ if that example is not annotated by the expert. In the above LePU setting, we have $l_i=1$ for $\{X_i\}_{i=1}^N$ and $l_i=0$ for $\{X_i\}_{N+1}^M+N$. The importance of introducing $l$ will become more clear in the next section.

Annotation process is an important part of data collection and can have a huge impact on the outcome and inferences of a learning algorithm that is trained based on such data. In particular in many data collection scenarios, collecting annotated data can be significantly harder/more expensive than acquiring unlabeled data. In such settings one might naturally face the LePU problem. Take the example of authenticity of online reviews. Such reviews today play an important role in people’s choice of products. In many cases businesses even hire people or create bots to generate fake reviews and ratings. One important problem then is to identify fake/deceptive reviews from authetntic ones. Based on our notation, in this case we can take $X_i$ to be the text of the review, $y_i$ to be the true class of the data point $X_i$ being deceptive (positive) or not. In some cases it is possible to find evidence to know that a given review is deceptive.\footnote{For example information such as multiple reviews from one IP address at the same time for different geographical locations.} But it is extremely difficult on such a platform to annotate an example as truthful; this is because by annotating an example as positive we have ruled out all the possible reasons that a given example is not deceptive/fake.\footnote{Such a tedious process would mean tracing the person back and making sure they had actually an experience of being in a particular restaurant/hotel or a similar commercial center that's posted on Yelp.com.} Despite the fact that there is a scarcity in the amount of available annotated data points in the case of detecting deceptive reviews, autonomous algorithms have been proposed that seem promising to solve this problem \citep{mukherjee2013yelp, jindal2010finding, ott2013negative}. Also more than 150 million reviews are available on Yelp.com of which mostly are unannotated. So if we could somehow use unlabeled data to leverage learning we would accomplish a lot considering the abundance of unlabeled data in such a domain.

Another similar situation is when we are studying a presence/absence dataset. In these situations ecologists are interested in understanding the geographical distribution of a given species in a given area. The data collection usually is a process that would lead to labeling the presence and absence of a species in a given geographical area. But ruling out all the areas from existence of that certain species would need an exhaustive search of every location in a given area, which can be practically impossible \citep{peterson2011ecological,ward2009presence, phillips2008modeling}. Needless to say high resolution information is available thanks to advances in geographic information systems and collecting ``unlabeled data'' in this case can also be done without a hitch. So again if we somehow incorporate the unlabeled data in our learning process there is a great potential to be able to improve the state of the art learning algorithms in this problem domain. 

The importance that learning under these extreme conditions and the ubiquity of facing such a situation in real-world problems has recently attracted a lot of attention to the LePU problem. In 
the next section we briefly discuss the related work on this problem.
\section{Related Work}
In machine learning One-Class Classification (OCC) \citep{moya1993one} is the problem of learning a concept/class and being able to differentiate it from other classes, with a training set consisting only of the elements of this specific class. OCC is also known under terms such as novelty detection \citep{bishop1994novelty}, outlier detection \citep{ritter1997outliers} and concept learning \citep{japkowicz1999concept}, depending on the area of use and application (for a general taxonomy see e.g., \citep{khan2009survey}).  

At first glance it might seem that LePU is a special case of OOC, but at least for the case where they are truly only two classes, virtually in all settings, unlabeled data is also available. So analogous to the extension of learning from supervised learning to semi-supervised learning, it is indeed a reasonable idea to use available unlabeled examples to leverage learning from data which bring is to the case of learning from positive and unlabeled data. An important challenge of using unlabeled data in many settings is the non-uniformity of the  distribution of unlabeled data. For example if we're interested to train a classifier to identify homepages from non-homepage webpages (e.g., as is done in \citep{yu2002pebl}), it is hard to collect unlabeled data, i.e., a corpus of webpages in general that could be uniform in every aspect and away from human bias \citep{khan2009survey}. In their seminal work, Elkan and Noto \citep{elkan2008learning}  explicitly formalize the lack of selection bias in collected data, given the class of the example. This in fact means they choose to neglect this non-uniformity and human bias in the annotation process. Then they introduce a Semi-Supervised Learning (SSL) algorithm for OCC in binary case and prove the identifiability of $p(y)$ and consequently $p(y|x)$. Before describing their work in more detail and as a reminder, in Section \ref{sec:intro} we introduced a random variable $l$, which indicates whether a given example $X_i$ with the class $y_i$ is annotated ($l_i=1$) or not ($l_i=0$). Now we are ready to introduce the Assumptions \ref{ASS1} and \ref{ASS2} incorporated in \citep{elkan2008learning}.
\begin{assumption}[No False Positive (NFP) assumption] $p(y = 1|l = 1) = 1$. This condition simply means all labeled examples belong to only the positive class. This is equivalent to assuming that there is no ``false positive'' labeling by the expert when it comes to finding the examples belonging to the positive class. Notice the use of double-quotes here: what we mean by no false positive is that if the expert is ``hunting'' for examples from the positive class, whatever they annotate will belong to the positive class, and as such their positive example detection never fails. And this, in the literature of signal detection theory, would mean they have no false positive (or that they have perfect precision) and whatever they detect indeed belongs to the positive class.
\label{ASS1}
\end{assumption}
\begin{assumption}[Selected Completely At Random (SCAR) assumption] $p(l = 1|y = 1, X = x) = p(l = 1|y = 1)$ or put it otherwise, $X$ is conditionally independent of $l$, given $y$. Elkan et. al. \citep{elkan2008learning} named this assumption as the \textit{``sampled completely at random assumption''}, meaning that the set of positive samples that are revealed to the learning algorithm are chosen randomly and identically from the set of all positive samples. As can be seen, the SCAR assumption exactly ensures that the way positive examples are chosen to be annotated are independent from their features which means the annotation process of positive examples is ``unbiased''. 
\label{ASS2}
\end{assumption}
Using the NFP assumption, Elkan et. al \citep{elkan2008learning} show that 
\begin{gather}
p(l=1|X)=p(l=1|y=1, X)p(y=1|X).
\label{eq:main_eq}
\end{gather}
Additionally from SCAR we get that $p(l=1|y=1, X)=c$ is a constant as a function of $X$, and as such from  (\ref{eq:main_eq}) it follows that \begin{gather}p(l=1|X)/c=p(y=1|X).
\label{eq:elkan_c}
\end{gather} 
In \citep{elkan2008learning}, the authors proceed with learning $p(l|X)$. Then they estimate $c$ through a holdout sample, and finally derive $p(y|X)$ through (\ref{eq:main_eq}). Since all three terms in (\ref{eq:main_eq}) are functions of $x$, for  simplicity and as a convention we will refer to these posteriors with the following short-hand interchangeably. We chose $s(x)$ for $p(l=1|X=x, y=1)$ so that $s$ stands for the initial letter in ``selection process''. Similarly we chose $t(x)$ for $p(y=1|X=x)$ where $t$ is the initial letter of ``target'' function, since the class is sometimes also known as target variable in machine learning literature. Finally we represent $p(l=1|X=x)$ with $h(x)$. Assumption \ref{ASS2} therefore is equivalent to $s(x) = c$, where $c$ is a constant independent of the values that $X$ attains.

Despite promising mathematical --and sometimes practical-- results under assumptions \ref{ASS1} and \ref{ASS2}, it is important to note that assumption (\ref{ASS2}) can be
highly violated in real-world problems. For example in the case of deceptive reviews in fact human judges suffer from certain biases in identifying truthful/genuine reviews (see \citep{ott2013negative} or  \citep{vrij2008detecting} for a more explicit study), which will in turn make them more likely to annotate certain examples more frequently (possibly the ones they are confident in identifying them as deceitful) than others. Another example as we discussed previously is a homepage classifier because of the hardness of collecting a corpus of webpages that is diverse and unbiased. In fact such bias definitely also affects the choice of positive examples, i.e., the way an expert would navigate webpages to find homepages to annotate them.

Therefore, we believe that $p(l=1|y=1, X=x)$ is not constant (w.r.t. to $X$) --which is assumed in the SCAR assumption-- but rather dependent on  $X$: labeling of examples in many cases is done by human experts and these experts will be sensitive (biased) towards specific features of samples. Our goal is to replace SCAR assumption with a more realistic one. One thing to notice is that both of the assumptions suggested by Elkan et. al. \citep{elkan2008learning} (NFP and SCAR) are about the data-generating process. Motivated by the same approach, namely focusing on the data-generating process, the goal of this work is to find a suitable (and in fact more flexible) assumption that would replace Assumption \ref{ASS2}. In the next section we will demonstrate how to still learn $p(y|X)$ through a dataset where only positive examples are annotated, but with a replacement of Assumption \ref{ASS2} with one that incorporates and exploits annotation process in a more meaningful way.

\section{Encoding Annotation Process in Learning}
When studying LePU, one can consider two different data-generation schemes from which the training set is collected from. In the more traditional setting, we assume the training dataset consists of two parts: One is the annotated part of the sample $\{X_i, Y_i\}_{i=1}^N\sim P(X, y)$, and the other is the unannotated part of the sample  $\{X_j\}_{j=N+1}^{N+M}\sim P(X)$. This is the so-called ``case-control'' scenario,\footnote{In the gathered dataset, the examples with a recorded label are called ``control'' or ``presence'' in this setting.} such as the one in \citep{ward2009presence}.

There is however another data-generating process that one can consider for LePU problem. As we discussed, in their data-generating process, Elkan and Noto introduce a new random variable $l$ which indicates whether a given example is annotated or not to incorporate the annotation process in their modeling. In fact in this non-traditional setting they assume data is generated according to a joint distribution $p(X, y, l)$. Therefore our dataset is a sample in the form of $x_i, y_i,l_i$ where $x_i\in\mathcal{X}$, $y_i \in \{0,1\}$, and $l_i \in \{0,1\}$, s.t. $y_i$ is observed only if $l_i=1$ and it is not observed otherwise ($l_i=0$). Moreover they assume that $l_i = 1$ \textit{only} when $y_i = 1$, which represents the fact that we only observe data points that belong to the positive class.

It is important to emphasize the difference between these two approaches in tackling the LePU problem, and why we choose to consider this non-traditional setting to solve this problem. As is shown in \citep{ward2009presence}, in the case-control scenario $p(y=1)$ is not identifiable given a dataset comprised of \textit{only} unlabeled data points and data points belonging to the positive class. However when we assume the data collection and annotation through the joint distribution $p(X,y,l)$, one might still be able to infer $p(y)$   \citep{elkan2008learning}. Additionally in this way we can incorporate how the class of an instance --together with its features-- plays role in whether a given example is annotated or not. This is in contrast to the case-control study, where such  a connection is stripped away of the available dataset. This is because we do not have access to the value of $l$ for our annotated data points $X_i, y_i$ with $y_i=1$. As an example to see the importance of annotation process and its dependence on the class of an instance and its features, we can again consider the presence-absence problem: whether a location is going to be labeled is both a function of how habitable that area is for the given specie and how easy to access that location is for the expert exploring a given territory. For these reasons and because we also are interested in incorporating annotation process as a way to improve learning, we will use this non-traditional data-generation scheme for our setting. 

\section{Modeling and Exploiting Annotation Process}
Considering the issues with the method of \citep{elkan2008learning} that were described in the previous section, our goal is to present a more realistic representation of the annotation process, i.e. $p(l=1|y=1, X)$. In fact one intuitive interpretation of this posterior is \textit{``what is the propensity of a given positive example with features $X=x$ to be selected and annotated by an expert?''}. With such an interpretation it seems only natural to assume this selection function is continuous w.r.t the features of a given instance. For example in the presence-absence problem, if location $X_i$ is annotated by an expert, it is very likely that the locations close to $X_i$ would also get annotated by an expert. Another example is medical imaging: Diagnosing a cancerous tumor through imaging techniques is a challenging task for physicians, and the success of the physician --among other things-- is dependent on qualities of the image, such as the resolution, the size/shape of the suspicious abnormality in tissues, etc. It is again reasonable to assume the odds of success for a physician to identify presence/absence of a tumor, or it being malignant/benign thereof, is a smooth function of such qualities of the image. This motivates Assumption \ref{pos:smoothness} which is at the core of our modeling of the annotation process. 
\begin{assumption}[Smoothness assumption]
The decision/selection function $s(x)=p(l=1|X=x, y=1)$, i.e. which element of the positive class an expert chooses to be annotated, is dependent on certain  features of the instance, $x$, and it is a smoother function than $t(x)=p(y=1|X=x)$, the posterior of the instance $X=x$ belonging to the positive class.
\label{pos:smoothness}
\end{assumption}
\vspace{-.2cm}
 Smoothness above can be defined more rigorously depending on the choice of models to represent $t(x)$ and $s(x)$. No matter what choice of smoothness assumption is made though, this is a strong relaxation of the SCAR assumption in \citep{elkan2008learning}. To see this, note that $s(x)$ in \citep{elkan2008learning} is a constant function, and therefore smoother than any non-constant function. This is assuming that $t(x)$ is not constant function itself, since otherwise it means $y$ is not learnable and independent from $X$.  As we discussed the next step is to choose a proper class of functions to represent $s(x)$, and accordingly $t(x)$. We know that $p(l=1|X=x, y=1)$ indicates the choice/decision function for the expert as to \textit{``which example to choose to annotate"}. In such an interpretation $p(l=1|y=1)$ encodes the sensitivity of detecting positive examples for our expert.\footnote{In fact in a probabilistic sense $s(x)$ is the \textit{instance-based sensitivity} as it is conditioned on a given instance $X=x$. And the aforementioned sensitivity is the average of it over $p(X|y=1)$.}
 In the literature of psychophysics, such a sensitivity is modeled  with a family of functions known as \textit{``psychometric functions''} \citep{prins2016psychophysics, wichmann2001psychometric}.

 \subsection{Psychometric Functions}

Psychophysics quantitatively investigates the the property of human perception, by assessing human response/reaction to a given sensory stimulus. When the human response is confined to two possibilities, there are generally two settings for experiments in psychophysics. One is known as the 2 Alternative Forced Choice (2AFC) experiment scheme. And the other is the yes/no experiment setting. Both of these settings can be modeled using the following family of functions known as \textit{``psychometric functions''}:
\begin{align}
\Psi(x;\alpha,\beta,\gamma,\lambda)=
\gamma+(1-\gamma-\lambda)F(x;\alpha,\beta)
\label{eq:psychometric}
\end{align}
where $F$ is some element of the exponential family distributions. It is assumed that $\alpha$ and  
$\beta$ are the parameters characterizing the underlying sensory mechanism \citep{prins2016psychophysics}. The other two parameters, namely $\gamma$ and $\lambda$, do not have to do with the underlying sensory mechanism and they characterize chance performance and lapsing by the subjects \citep{prins2016psychophysics}. The parameter $\gamma$ is known as \textit{``guessing rate"}, which is a bit of a misnomer, as there are theories that human subjects never truly guess \citep{prins2016psychophysics}(see Section 4.3.1), but since it is conventional to call it guessing rate, it will be called so here as well. The idea is that if the human subject was to guess, there is a certain probability that their guess would be right, and this probability is encoded with the guessing rate $\gamma$. For example in a 2AFC performance-based task the subject is presented with two stimuli, and they are supposed to indicate which option is the one that carries the stimuli. Usually this is done by providing a pair of buttons where each of two buttons is associated with one of the options. Then by pressing the button for any given task, a subject indicates where they believe the stimuli is presented. As such, the goal is to see if the subject can identify the signal, when is forced to make choice between two offered options. In such a case, i.e., in 2AFC, it is usually assumed $\gamma$ is  $0.5$. This parameter basically defines the lower bound of the psychometric function, as can be seen in Figure \ref{fig:guessing_rate}. It is assumed that $\gamma$ belongs to $[0,1]$, $[0, 1)$ or $(0, 1)$ \citep{wichmann2001psychometric}. In what follows we assume $\gamma\in(0, 1]$.
\begin{figure}[H]
    \centering
    \includegraphics[width=0.9\linewidth]{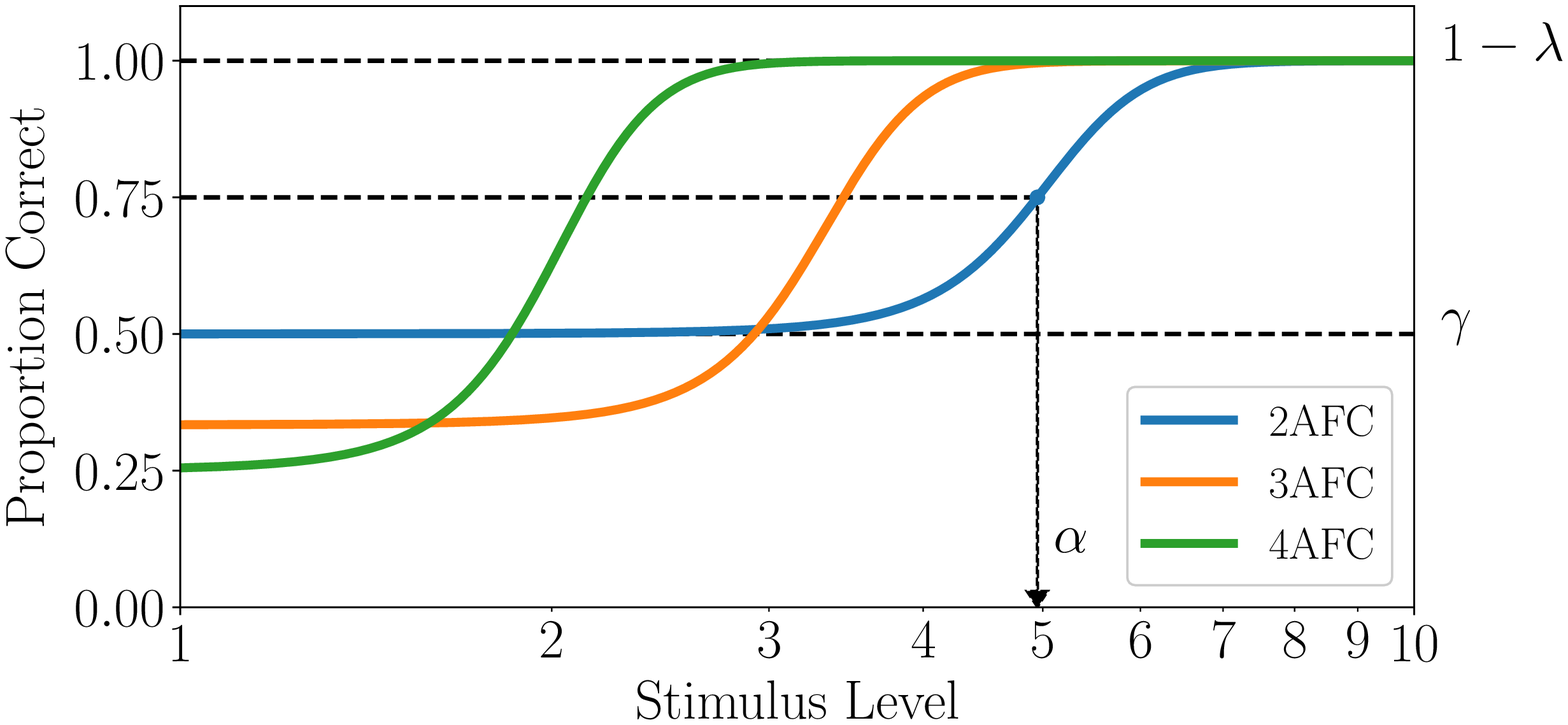}
    \caption{The guessing rate $\gamma$ for  Alternative Forced Choice tasks with different number of choices where it is mostly assumed $\gamma=1/n$ for $n$ alternative AFC or $n$-AFC, for short.} 
    \label{fig:guessing_rate}
\end{figure}
Finally $\lambda$ represents the so-called \textit{``lapse rate''}. The idea is that even when the stimulus strength is in a range that the subject would be able to detect it, sometimes they cannot do so simply due to a memory lapse, a sneeze during the experiment process, etc. This will introduce a small error in detection --despite the strong available signal-- and as such $1-\lambda$ defines an asymptote from above for $\Psi$ (similar to $\gamma$ which was lower-bounding the psychometric function). $\lambda$ is usually assumed to be small. Here we assume $\lambda\in[0, 1)$. We additionally assume that $\gamma+\lambda\leq 1$, as is done in the literature of psychopysics \citep{wichmann2001psychometric}.
\subsection{Two Classes of Models}
Now that we defined the psychometric function and gave interpretations of it in the context of psychophysics, here we describe in more detail how we will incorporate such a function for our purposes. We will take $F$ in (\ref{eq:psychometric}) to be a sigmoidal function, i.e. $F(x;\alpha, \beta)=\sigma(x;\alpha,\beta)=\sigma(\alpha^Tx+\beta)$ where $\sigma(x):=(1+\exp(-x))^{-1}$. Then
\begin{gather}
p(l=1|X) = \Psi(x;\alpha,\beta,\gamma,\lambda) \sigma(x;a,b).
\label{eq:psychm_posterior}
\end{gather}
where  $\sigma(x;\alpha, \beta):=\sigma(\alpha^Tx +\beta)$. We will be considering two settings for the above model. In the first setting we will set $\gamma=\lambda=0$. Then $p(l=1|X)$ will be the product of two sigmoidal functions. As such we call this first model Sigmoidal Product Model (SPM) with symbol $\mathbf{\Sigma}$ representing the family of such parametric models. In the second setting we will be considering $\gamma$ and $\lambda$ to be free parameters beside $\alpha$ and $\beta$, and this second family is  represented with $\mathbf{\Psi}$. We will call the elements of $\mathbf{\Psi}$ PsychM, as a short form for Psychometric Model. Notice that SPMs are a subset of PsychM family, where $\gamma=\lambda=0$. Additionally it is easy to see that Elkan et al's \citep{elkan2008learning} family of models described in (\ref{eq:main_eq}) is a special case of SPMs (with $\alpha=0$ and $\beta=-\log(c/1-c)$) and PsychMs (with e.g., $\gamma=1-\lambda=c$). Also for brevity, we will sometime refer to all the parameters of the $s(x)$ with $\theta_s$ and all the parameters of $t(x)$ with $\theta_t$.  To infer the model parameters, in both cases we will maximize the conditional log-likelihood function of the observed variables, i.e. $\{l_i\}_{i=1}^N$ conditioned on $\{X_i\}_{i=1}^N$ using (\ref{eq:psychm_posterior}) as follows
\begin{align}
\mathcal{L}\mathcal{L} &= \log  p(\{l_i\}_{i=1}^N|, \{X_i\}_{i=1}^N, \theta_s, \theta_t)\label{eq:conditional_ll}\\
&=\sum_{i=1}^{N} \left[\vphantom{\biggr.}l_i\log p(l=1|X=X_i, y=1, \theta_s)\right.\notag\\ 
&~~~+ \left. l_i\log p(y=1|X=X_i, \theta_t) \right.\notag\\
&~~~~~~+ \left. (1-l_i)\log \Big(1- \big[p(l=1|X=X_i, y=1, \theta_s)\right. \notag\\
&~~~~~~~~~\times \left. p(y=1|X=X_i, \theta_t)\big]\Big)\vphantom{\biggr.}\right].\notag
\end{align}
Derivation of (\ref{eq:conditional_ll}) is available in supplementary material. Note that in the case of SPMs, $\theta_s=(\alpha,\beta)$, whereas in the case of PsychMs it is $\theta_s=(\alpha,\beta,\gamma,\lambda)$. We can use MLE to estimate the parameters of the families of the model above, since both of these families are identifiable under some mild conditions as it is shown below.
\subsection{Identifiablity}
In what follows we will investigate the identifiability of PsychMs and SPMs. In fact we will show that the class $\mathbf{\Sigma}$ is identifiable up to a permutation of parameters; notice that in fact the product of two sigmoidal functions do not change when we swap  their parameters, i.e.
\begin{restatable}[Identifiability of SPM]{theorem}{spmidentifiability}
\label{thm:SPM_identifiability}
Assume that the support of $X$ is $\mathbb{R}^n$, where $n\geq 1$. Then $\mathbf{\Sigma}$ is identifiable, i.e. for any set of parameters $(\alpha, \beta, a,b)$ and $(\alpha', \beta', a',b')$ if it is the case that 
\[
\sigma(x;\alpha, \beta)\sigma(x;a, b) = \sigma(x;\alpha',\beta')\sigma(x;a', b'),
\]
it follows that
$(a, b, \alpha, \beta)=(a', b',\alpha', \beta')$ or $(\alpha, \beta, a, b)=(a', b',\alpha', \beta')$.
\end{restatable}
The proof of this theorem is presented in the supplementary material. We will show $\mathbf{\Psi}$ is \textit{``generically identifiable''}. We call a model class generically identifiable if it is identifiable almost everywhere  in terms of Lebesgue measure \citep{allman2009identifiability}. In fact in this case the measure zero set where the parameters of $\mathbf{\Psi}$ are not identifiable are when $\gamma+\lambda=1$, $\gamma=0$, $\lambda=0$ and $\alpha=0$. As such from this point on we will assume $\mathbf{\Psi}$ only contains PsychM models with $\alpha\neq 0$, $\gamma+\lambda<1$ and $\gamma > 0$. Additionally our proof uses the assumption that $X\in\mathbb{R}^n$ with $n>1$. The case of $n=1$ would need further consideration.
\begin{restatable}[Identifiability of PsychM]{theorem}{psychmidentifiability}
Assume that the support of $X$ is $\mathbb{R}^n$ s.t. $n\geq 2$. Then for $h(x;\alpha, \beta,a,b,\gamma,\lambda)\in\mathbf{\Psi}$ and $h'(x;\alpha', \beta',a',b',\gamma',\lambda')\in\mathbf{\Psi}$ s.t. $h(x)=h'(x)$,
\[(a, b, \alpha, \beta, \gamma,\lambda)=(a', b',\alpha', \beta', \gamma',\lambda').\] 
\label{thm:PsychM_identifiability}
\end{restatable}
The proof of Theorem \ref{thm:PsychM_identifiability} is presented in the supplementary material.

\section{Algorithms}
As previously discussed, our algorithms for both SPM and PsychM models are based on MLE estimators. Additionally we introduce a novel way of regularization for our  models; an important part of our work is to pay a closer attention to the data-generating process by separating the modeling of $s(x)$ and $t(x)$ from each other. With doing so, we get the ability to separately regularize these two parts of the whole model. To describe the situation in more detail let us consider the parameters of PsychM model (the argument for SPM is similar), which are $\theta_s=(\alpha, \beta, \gamma,\lambda)$ and $\theta_t=(a, b)$. The weight vector is part of the logistic regression function that is regularized to enforce a certain structure on the inferred model, such as sparsity of the weight vector, which in turn would reduce the number of active features in the learnt model. Since the classification function $t(x)$ and the selection function $s(x)$ are modeling two separate processes, we will penalize the weight vector for them independently.  As such our loss function for minimizing the likelihood will be as follows:
\begin{gather}
l(\mathcal{D}; \theta_s, \theta_t) = -\mathcal{L}\mathcal{L} + C_\alpha \|\alpha\|_2^2 + C_a\|a\|_2^2,
\label{eq:main_loss}
\end{gather}
Where we incorporated two regularization coefficients for two separate parts of the model. The idea is that we need to penalize the weight vector of selection function $s(x)$ in a possibly different manner in comparison to penalizing the weight vector of the classification function $t(x)$. The choice of $C_\alpha$ and $C_a$ then can be done using usual methods in model selection such as cross-validation. Notice that we can use other norms in (\ref{eq:main_loss}) depending on our prior understanding of the weight vector for each process. For example, depending on the feature vectors, it might be reasonable to assume humans --who are the experts annotating the data in most cases-- will use a smaller set of features than the real classification function and as such we might replace $\|\alpha\|_2^2$ in (\ref{eq:main_loss}) with $\|\alpha\|_1$, to ensure sparsity of $\alpha$. 

Now that we have defined the main loss function, we proceed to describe our learning algorithms for these two types of models. We divide this section into two subsections describing the algorithms for each class separately.
\subsection{Learning Parameters of SPM}
Here we present the algorithm for learning the parameters of the SPM model. The idea is to maximize the marginal likelihood $p(l|X)$ and choose among the learned parameters $\theta_1$ and $\theta_2$ the one with larger norm as $\theta_t$, i.e. we would choose $\theta_1=\theta_t$ if $\|\theta_1\|_2>\|\theta_2\|$ and otherwise $\theta_2=\theta_t$. This choice is motivated by Assumption \ref{pos:smoothness}, since the sigmoidal function with the smaller norm will be smoother than the sigmoidal function with the larger norm. Algorithm \ref{alg:SPM_SGD} presents the psuedocode for this algorithm.
\begin{algorithm}[!h]
\DontPrintSemicolon
\caption{SPM learning algorithm}
\label{alg:SPM_SGD}
\tcp{Estimating $\theta_s$ and $\theta_t$}
Use L-BFGS \citep{nocedal1980updating} (for low dimensions) or Nadam \citep{dozat2016incorporating} the loss function $l$ in (\ref{eq:main_loss}) and learn $\theta_1$ and $\theta_2$, the parameters of the SPM.\\
\tcp{Choosing the right permutation}
Choose whether $\theta_1=\theta_s$ or $\theta_2=\theta_s$ based on Assumption \ref{pos:smoothness} to be "if $\|\theta_1\|_2<\|\theta_2\|_2$ then $\theta_1 = \theta_t$". This is because the sigmoidal function that has parameters with larger norm is a steeper one and therefore less smooth.\\
\tcp{Returning the classifier}
The sigmoidal model with $\theta_t$ as its parameter is used for classification.
\end{algorithm}
\subsection{Learning Parameters of PsycM}
The implementation of algorithm for PsychM was a bit more challenging, considering that the conditional log-likelihood function presented in (\ref{eq:conditional_ll}) is a non-concave function. Additionally in this case we are dealing with a constrained optimization problem due to the fact that $\gamma$ and $\lambda$ are bounded variables. Although algorithms like Sequential Quadratic Programming (SQP) exist for constrained optimization (See e.g. \citep{nocedal2006sequential}, Chapter 18) , for high dimensions such algorithms suffer from the curse of dimensionality as they rely on the calculation of the Hessian. Application of interior methods such as the Barrier method (See e.g. Chapter 19 of \citep{nocedal2006sequential}) did not help in practice. For that reason to enforce the constraints on $\gamma$ and $\lambda$ we reparameterize them with
\begin{gather}
\gamma=\frac{|{\gamma'}|}{1+|{\gamma'}|+{|\lambda'}|}\ {\rm and\ } \lambda=\frac{|{\lambda'}|}{1+|{\gamma'}|+|{\lambda'}|}.
\label{eq:reparam}
\end{gather}
Then we use Algorithm \ref{alg:psychm_algorithm} to learn the parameters of PsychM.
\begin{algorithm}[!h]
\setcounter{AlgoLine}{0}
\DontPrintSemicolon
\caption{PsychM learning algorithm}
\label{alg:psychm_algorithm}
\tcp{Estimating $\theta_s$ and $\theta_t$}
Use Adam to minimize the loss function defined in (\ref{eq:main_loss}) after reparametrizing it using (\ref{eq:reparam}), and learn $\theta_t$ and $\theta_s$, the parameters of the PsychM.\\
\tcp{Returning the classifier}
Use the learnt $\theta_t$ as the parameter vector for the sigmoidal function representing the classifier.
\end{algorithm}

\section{Experimental Results}
In what follows we compared the performance of our algorithms to two other LePU algorithms, one of which is the method proposed in \citep{elkan2008learning}, and the other is what we call the ``Na\"ive method" to be introduced below. We also include an unrealistic method to show the limit of the performance of any method. We present this results breaking it down to performance with synthetic data and performance with real-world data. 
\subsection{Experiments with Synthetic Data}
\label{subsec:synthetic}
To evaluate the success of our algorithms empirically, we have generated a dataset of $\{(l_i, y_i, X_i)\}_{i=1}^N\sim P(l, y, X)$, with 
\[
P(l, y, X) = p(l|y,X)p(y|x)p(X),
\]
where $p(l|y=1, X)$ is chosen to be a PsychM model with 
\[
a,\alpha\sim \mathcal{N}_{5}(0, \rho_1^2I_{5}) + KR_{5}.
\]
Here $I_{5}$ is the 5 dimensional identity matrix and $R_{5}$ is the 5-dimensional Rademacher random vector, i.e., $R_i=1\ {\rm or} -1$ with probability $1/2$. Radmchaer random vector is added to the multivariate normal distribution to randomly shift the centers of the multivariate  Gaussian distributions that $\alpha$ and $a$ are chosen from. Finally $b,\beta\sim N(0, \rho_2^2)$. In the following experiment we chose $N=5000$, $\gamma=\lambda=0.05$. $\rho_1=10$, $\rho_2=1$ and finally $K=5$. 
For comparing these models in both of the above cases, we split the dataset to training and test sets of equal size. For any such pair of sets we train 5 models on the training set. These models are as follows:
\begin{enumerate*}[label=(\roman*)]
\item Elkan et. al's model \citep{elkan2008learning}, where we choose $p(l|X)$ to be a sigmoidal function. \item The ``Na\"ive classifier'' is basically a classifier where we apply logistic regression to the training set comprising of positive examples, and assuming that unlabeled data belong to the negative class. \item The ``Real classifier'' is trained given the full access to the class each example belongs to ($y_i$'s)--note that this method is not realistic, but is shown to indicate the limit in the performance of any practical method, because in reality we do not have access to $y_i$'s, what makes our problem particularly challenging. And finally \item SPM models and \item PsychM models as introduced previously. \end{enumerate*} For hyperparameter tuning in regularization of all these methods we use 3-fold Cross-Validation (CV), where Brier score \citep{brier1950verification} is used to choose the most suitable model in classifying $(X_i, l_i)$ pairs. For any score-based classifier $M$ with score $M(X)$ over instance $X$, Brier score is just the mean squared error over the a given dataset $\mathcal{T}$. i.e.,
\[
Brier(M)=\frac{1}{|\mathcal{T}|}\sum_{(X_i, y_i)\in\mathcal{T}}(M(X_i)-y_i)^2.
\]
It is well-known that Brier score can be used to choose well-calibrated models, i.e. models such that the score $M(X)$ is in fact a good estimation of $p(y=1|X)$ for the true distribution $p(.|X)$. For the model selection through trial and error among possible score choices including area under the ROC curve (the higher, the better), Average Precision Curve (the higher, the better), and Brier score (the lower, the better), Brier score had the best success in finding a good model to approximate $p(y|X)$, and we chose it for that reason.

Then we compare the performance of these methods on the test set comprising of $(X_i, y_i)$ pairs in terms of classification scores such as the $f1$ score, classification accuracy, and also  Brier score. We repeat these random trials 500 times and  report the results in Table \ref{tab:synthetic}, with the best values depicted in \textbf{bold} and the second best depicted in \textit{italic}. 
Since the results by the ``Real classifier'' is not achievable by any LePU method, they are not highlighted.
\begin{table}[h]
\label{tab:freq}
\centering
\scalebox{0.8}{
  \begin{tabular}{|p{1.1cm}|c|c|c|c|c|l}
    \toprule
    & SPM & PsychM & Na\"ive & Elkan & Real\\
    \midrule
    f1 & {\bf 0.8920} & {\it 0.8876} & 0.8807& 0.8826 & 0.9598\\
    \midrule
    AUC & {\it 0.9748} & {\bf 0.9761} & {\it 0.9748}& 0.9735 & 0.9931\\
    \midrule
    test acc. & {\bf 0.9023} & {\it 0.8989} & 0.8931& 0.8946 & 0.9605\\
    \midrule
    Brier & {\bf 0.0779} & {\it 0.0790} & 0.0900 & 0.0887 & 0.0304\\
\bottomrule
\end{tabular}
}
 \caption{Average score of methods on 500 trials described above.}
\label{tab:synthetic}   
\end{table}

One can see that  SPM and PsychM outperform all other methods on all the scores. Here, however, we also perform a test to assure the significance of our results. Notice that all the score functions above are applied to a test set, denoted by  $\mathcal{T}$. Now for any model $M$ and any score function $S$, $S_{\mathcal{T}}(M)$ is the score of model $M$ under the test set $\mathcal{T}$. Now for any two models $M_1$ and $M_2$ and test set $\mathcal{T}$, $D_{\mathcal{T}}(M_1, M_2):=S_{\mathcal{T}}(M_1)-S_{\mathcal{T}}(M_2)$ will measure the respective success of $M_1$ to $M_2$ over $D$ using $S$. So in these 500 trials we have test sets $\{{\mathcal{T}}\}_{i=1}^{500}$, calculate $S_{\mathcal{T}}^i(M_1)-S_{\mathcal{T}}^i(M_2)$, and consider its empirical distribution. We say $M_1$ is $\alpha$-significantly better than $M_2$ if the $\alpha/2$-quantile of empirical distribution of values $T_D^i(M_1, M_2)$ is non-negative. We chose the significance level of $\alpha=0.9$. Table \ref{tab:significance} shows the results of this cross-comparisons, where for row $i$ and $j$, we have calculatec $D_{\mathcal{T}}(M_i, M_j)$ as defined above and if the $\alpha/2$ quantile of the empirical score distribution was larger than zero the results are deemed significant (where we put a checkmark at that location of the table below). 
\begin{table}[H]
\centering
\scalebox{0.8}{ 
  \begin{tabular}{c|cccccl}
    \toprule
     & PsychM & Na\"ive & Elkan & Real\\
    \midrule
    SPM &   \checkmark & \checkmark& \checkmark & \checkmark\\
    \midrule
    PsychM    &  & \xmark& \xmark & \xmark\\
    \midrule
     Na\"ive   &  & & \xmark & \xmark\\
    \midrule
    Elkan &   & & & \xmark\\
\bottomrule
\label{tab:significance}
\end{tabular}
}
 \caption{The comparison of the f1 score as described above for every two pair of models. Here the training dataset is small and $\gamma=0.05$ (small) and $\lambda=0$, so the PsychM method and SPM are close to each other; in the result SPM outperforms PsychM because it is simpler (involving fewer parameters).}
\label{tab:significance}   
\end{table}
As can be seen in Table \ref{tab:significance} the f1 score for SPM is significantly better than other methods in this setting of the parameters. For larger values of $\gamma$ and $\lambda$ PsychM turns to outperform all the other methods. And as such when the sample size is significantly large both SPM and PsychM do better than Elkan's methods and the Na\"ive method. Notice that this is expected since both of these models are special cases of PsychM models, and SPM. To realize Elkan's method, we need to take the selection function to be the constant $c$ in (\ref{eq:elkan_c}) and to realize the Na\"ive method we just need to set this constant parameters of selection function such that $c=1$. 
\subsection{Experiments with Real-World Data: Neuroscience}
\label{subsec:real_world}
The SCAR assumption is explicitly assumed to create the final dataset \citep{bekker2018learning, ward2009presence}. Here we introduce a family of datasets that are well-suited for the learning and assessment of LePU methods, and introduce a general method to create datasets for LePU based on experimental data ubiquitously available online.

In recent years there has been an abundance of datasets on human perception, recognition, and assessment on different visual/auditory tasks. Such datasets in many cases can be divided into correct, incorrect, and undecided assessment by humans. Since the ground truth in these cases is known, these datasets seem to be a good candidate to be dealt with LePU learning and our methods.
We focus here on a dataset which is presented in  \citep{delorme2004interaction}.\footnote{The dataset itself can be downloaded from \url{https://sccn.ucsd.edu/~arno/fam2data/publicly_available_EEG_data.html}.} We briefly describe the experimental paradigm of this dataset. 14 subjects (7 male, 7 female) participated in a study where they are performing a go/no-go categorization task. Here we focus on a sub-task of this study where the subjects had to decide if there is an animal in the shown picture/stimulus or not by pressing either of two possible keys. There were 10 blocks of trials, with 100 pictures in each block. In each block an equal number of animal pictures and non-animal pictures are shown to the subject, and the stimulus exposure was confined to 200ms. During all the block trials the brain activity of the subjects was recorded using EEG brain imaging technique. For the purpose of our work the EEG activity was not relevant and as such will be omitted. 

The dataset can be summarized in the format of $(X, y_{sub}, y_{real})$ where $X$ is the image displayed to the subject, $y_{sub}$ is the class given by the subject to $X$, and $y_{real}$ is the real class of the instance $X$. The subject needs to decide for any image $X$, whether it is an animal picture ($y_{real}=1$) or not ($y_{real}=0$). Knowing that we have only two classes, we let $l=y_{real}\land y_{sub}$, where $\land$ is the logical AND operator. This in a way enforces the condition that subjects only classify positive examples, i.e. Assumption \ref{ASS1}.\footnote{Unfortunately a two-alternative forced choice or go/no-go experimental task designs are pretty common in psychophysics and neuroscience and finding dataset that subjects are allowed to be indecisive is uncommon.} Notice that here the data relevant to EEG recordings are discarded. 

Because for image classification the dataset is quite small, we used a pretrained neural network known as VGG16  \citep{simonyan2014very} as an initial feature extractor for our task. We pass our initial images through this pre-trained neural net and take the activation of the first fully connected layer of this network as the feature set for all of our pictures in the dataset, and $X_i$ below refers to this embedding using VGG16.

After this pre-processing we applied LePU algorithms to these featurized datasets. The training-test split and the choice of models and model parameters are identical to the those used in Subsection \ref{subsec:synthetic}, and their regularization coefficients are also chosen in identical fashion, with the only exception that $C_\alpha$ here is penalized with $l_1$ regularization, and initial $\gamma$ and $\lambda$ in optimization were set to $0.7$ and $0.02$, respectively, for the PsychM model. This was only done to increase the convergence speed of PsychM model. Then we compare the success of different LePU methods over the test set $\{(X_i, y_i)\}$'s. We chose five of the subjects for which the classification error for subjects were the highest. This is done mainly because when human accuracy is really high (say, above $\%90$), it implies $s(x)$ is also really high, potentially making it close to constant. Therefore technically, the assumption of Elkan et al. (i.e., $s(x)$ being constant) holds, and their method provides a good estimation. We refer to the subjects with 3-letter abbreviations given in \citep{delorme2004interaction}. We compared the success of the five models previously mentioned over the five datasets for subjects with name encodings `fsa', `mta', `sph', `hth', `mba'. 

Here we report the test set accuracy of the five models over five subjects. But rather than reporting the results on one trial we applied Bootstrapping.   This is done because reporting the results over a single trial can be highly dependent on the train-test split and the sample size.  Considering that our features have 4096 dimensions and our dataset for each subject was only of size 1000, such an averaging seemed necessary. In fact this averaging has been introduced in machine learning literature \citep{jain1987bootstrap, duda2012pattern}, but it is rarely applied in practice. To ensure the soundness of our results, especially considering the high accuracy score across all the five models we applied the bootstrap method over 200 trials and reported the average here. As mentioned previously, the learning algorithms are set up similarly to Subsection \ref{subsec:synthetic}. So for every bootstrap resample of $1000$ data points for each subject, we divide this sample into training and test sets of an equal size and use the training set to train all 5 models and then compare their success on the training set for that particular bootstrap resample. Table \ref{tab:real_world} shows the average of the test set accuracy over 200 bootstrap resamples. Notice that the across all the subjects SPM and PsychM have the highest (depicted in \textbf{bold}) and second highest (depicted in \textit{italic}) average accuracy scores. From the table one can see that our methods, SPM and PsychM, perform best across all considered subjects. This result was consistently true for other classification metrics such as Brier score, Area Under the Receiving Operating Curve (AUCROC) and also F1 score. For space constraints these tables  are provided in the supplementary material (Table \ref{tab:Brier}, \ref{tab:AUC}, \ref{tab:f1}).
\begin{table}[h]
\centering
\scalebox{0.8}{  \begin{tabular}{|p{1cm}|p{1.2cm}|p{1.2cm}|p{1.2cm}|p{1.2cm}|p{1.2cm}|l}
    \toprule
    Subjects &  `fsa' & `mta' & `sph' & `hth' & `mba'\\
    \midrule
    SPM & $\textit{0.90387}$& ${\bf 0.92031}$ &$\textit{0.92633}$& ${\bf 0.93073}$ & $\textit{0.93840}$\\
    \midrule
    PsychM & ${\bf 0.90515}$ & $\textit{0.92025}$ & ${\bf 0.92811}$&  $\textit{0.93015}$ & ${\bf 0.93981}$\\
    \midrule
    Na\"ive  & $0.89486$ & $0.91154$ & $0.91915$ &  $0.92400$ & $0.93363$ \\
    \midrule
    Elkan & $0.89556$ & $0.91135$ &  $0.91862$ & $0.92237$ & $0.93241$  \\
    \midrule
    Real & $0.96021$ & $0.95958$& $0.95984$& $0.96025$ & $0.96016$\\
\bottomrule
\end{tabular}
}
 \caption{The average test set accuracy of different models across different datasets (corresponding to different subjects), in repeated 200 bootstrap resamples.}
\label{tab:real_world}   
\end{table}
\section{Conclusion and Discussions}
In this work we introduced a novel framework for learning from positive and unlabeled data, which builds on the previous study in \citep{elkan2008learning}, but importantly extends that work by eliminating the SCAR assumption, presented in Assumption \ref{ASS2}. We believe that SCAR is a very strong assumption, because the features of the elements belonging to the positive class (or similarly negative class) usually play an important role on how likely it is for these positive instances to be selected by the expert for annotation, as also suggested by our empirically results. Based on this idea and inspired by the human decision-making process in psychophysics, we introduced two family of models, namely Sigmoidal Product Model (SPM) and Psychometric Model (PsychM), which both take into account properties of the annotation process and enable us to learn them from the LePU data. 

We showed that under mild assumptions our introduced models are identifiable. We then proposed algorithms that learn the parameters of these models by maximizing the likelihood of observed data. We also demonstrated that our introduced models outperform two other LePU methods in experiments  with both synthetic and real-world data. Finally we introduced a rich family of real-world datasets that LePU can be applied to without introducing a synthetic annotation mechanism. From those datasets, one can easily create data with $(X_i, y_i, l_i)$ triplets to assess the performance of LePU methods. 
On simulated data the two proposed methods significantly outperform all alternatives.  On the real data, they perform better than all the others across all five considered subjects. 
Moreover, we believe that the parametric assumptions on the form of $p(y=1|X)$ can be eliminated, while the identifiability of the model can still be established under milder assumptions, but we will leave such an extension and the study of it as future work.

\appendix
\section{Supplementary Material}
In this section we first provide the derivation for (\ref{eq:conditional_ll}) and then present the proof for Theorems \ref{thm:SPM_identifiability} and \ref{thm:PsychM_identifiability}. We also report additional results (based on other classification scores) on the real-data experiment here.
\subsection{Derivation of Observed Conditional Log-likelihood}
Due to space constraints we present the derivation of log-likelihood under Postulate \ref{ASS1} here as follows:
{\tiny
\begin{flalign}
\mathcal{L}\mathcal{L} =& \log  p(\{l_i\}_{i=1}^N|, \{X_i\}_{i=1}^N, \theta_s, \theta_t)\label{eq:conditional_ll}&\\
=&\sum_{i=1}^{N} l_i\log p(l=1|X=X_i, \theta_s, \theta_t) 
\notag&\\ 
&~~+ (1-l_i)\log(1- p(l=1|X=X_i, \theta_t, \theta_s))\notag&\\
=&\sum_{i=1}^{N} \left[\vphantom{\underbrace{p(l=1|X=X_i, y=0, \theta_s)}}_{0}l_i\log\Big[p(l=1|X=X_i, y=1, \theta_s)p(y=1|X=X_i, \theta_t) \right.\notag&\\
&~~+ \left. \underbrace{p(l=1|X=X_i, y=0, \theta_s)}_{0} p(y=0|X=X_i, \theta_t)\Big]\right.\notag&\\
&~~~~+\left. (1-l_i)\log \biggr(1- \Big[p(l=1|X=X_i, y=1, \theta_s)\right. \notag&\\
&~~~~~~\times \left. p(y=1|X=X_i, \theta_t)\right.\notag&\\
&~~~~~~+ \left.\underbrace{p(l=1|X=X_i, y=0, \theta_s)}_{0}p(y=0|X=X_i, \theta_t)\Big]\biggr)\right]\notag&\\
=&\sum_{i=1}^{N} \left[\vphantom{\biggr.}l_i\log\Big[p(l=1|X=X_i, y=1, \theta_s)\notag\right.&\\
&\left.\times p(y=1|X=X_i, \theta_t)\Big] \right.\notag&\\
&~~+ \left. (1-l_i)\log \Big(1- \big[p(l=1|X=X_i, y=1, \theta_s)\right. \notag&\\
&~~~~\times \left. p(y=1|X=X_i, \theta_t)\big]\Big)\vphantom{\biggr.}\right] \notag &\\ =
&\sum_{i=1}^{N} \left[\vphantom{\biggr.}l_i\log p(l=1|X=X_i, y=1, \theta_s)\right.\notag&\\ 
&~~+ \left. l_i\log p(y=1|X=X_i, \theta_t) \right.\notag&\\
&~~~~+ \left. (1-l_i)\log \Big(1- \big[p(l=1|X=X_i, y=1, \theta_s)\right. \notag&\\
&~~~~~~\times \left. p(y=1|X=X_i, \theta_t)\big]\Big)\vphantom{\biggr.}\right].\notag
\end{flalign}
}
\subsection{Proof of Theorems \ref{thm:SPM_identifiability} and \ref{thm:PsychM_identifiability}}
There are some similarities in the proofs but when for the psychometric function $\gamma$ and $\lambda$ are not zero, proving identifiability becomes a harder task that requires an essentially different technique. As such the proofs are presented completely separated from each other.
\subsubsection[Proof of Identifiability of SPM model]{Proof of Theorem~\ref{thm:SPM_identifiability} (See page~\pageref{thm:SPM_identifiability})}
In this section we present the proof for Theorem \ref{thm:SPM_identifiability}. The proof is mostly based on the limiting conditions when $X$ approaches to positive and negative infinity.
\spmidentifiability*
\begin{proof}
 Without loss of generality if we can show the proof for one-dimensional case, we can conclude it for multidimensional. This is because if we set $x=e_i$ where $e_i$ is the $i$-th standard orthonormal basis for $\mathbb{R}^n$ in (\ref{eq:log_identifiability}) we will get the same equation only for the one-dimensional case. Now we have
\begin{align*}
&(\exp(-(ax+b))+1)^{-1}\times (\exp(-(\alpha x+\beta))+1)^{-1}
\\=&(\exp(-a' x+b'))+1)^{-1} \times (\exp(-(\alpha' x + \beta'))+1)^{-1}.
\end{align*}
Therefore
\begin{align*}
&(\exp(-(ax+b))+1) \times  (\exp(-(\alpha x+\beta))+1)\\
=&(\exp(-(a' x+b))+1) \\
&\times (\exp(-\alpha' x - \beta)+1).
\end{align*}
simplifying we get
\begin{align}
&e^{-(a+\alpha)x+\beta-b} + e^{-ax-b} + e^{-\alpha x-\beta}\notag\\
=&e^{-a'x-b'} + e^{-(a'+\alpha')x+\beta'-b'} + e^{-\alpha'x-\beta'} 
\label{eq:LRID}
\end{align}
Take 
\begin{align*}
f(x):=&e^{-(a+\alpha)x-\beta-b} + e^{-(ax+b)} + e^{-(\alpha x-\beta)}\\ 
-& e^{-(a'+\alpha')x-\beta'+b'} 
- e^{-a'x-b'}- e^{-\alpha'x-\beta'}.
\end{align*}
According to (\ref{eq:LRID}) we have $f(x)=0$ for any $x\in\mathbb{R}$. Also taking derivative of $f$  w.r.t. $x$ we get
\begin{align}
f'(x)=&-(a+\alpha)e^{-(a+\alpha)x-\beta-b}-ae^{-(ax+b)}-\alpha e^{-(\alpha x+\beta)}\\
+&(a'+\alpha')e^{-(a'+\alpha')x-\beta'-b'}
+a'e^{-a'x-b'}+\alpha'e^{-\alpha'x-\beta'}=0
\label{eq:diffLRID}
\end{align}
for any $x\in\mathbb{R}$.
We divide the proof into cases. 
\begin{enumerate}[label=(\roman*)]
    \item $a>0$ and $\alpha>0$:
This means $a+\alpha>0$. It follows that $a'\geq 0$ and $\alpha'\geq 0$, as otherwise taking the limit of $x$ to infinity leads to a contradiction as right hand side of (\ref{eq:LRID}) goes to infinity whereas left hand side of it approaches to a real value. This means $\alpha'+a'\geq \max\{\alpha', a'\}$. But this implies $\alpha+a=\alpha'+a'$ as otherwise there will be a dominating exponent in $f(x)$ and as a result 
\[
\lim_{x\to-\infty} f(x)=\infty\ {\rm\ or\ } \lim_{x\to\infty} f(x)=-\infty
\]
which is a contradiction since $f(x)=0$. Now it cannot be the case that $\alpha'=0$ or $a'=0$. Suppose to the contrary that this is the case. WLOG assume $\alpha'=0$. Divide both sides of (\ref{eq:LRID})  with $\exp(a+\alpha)x$ and take the limits to $-\infty$. From (\ref{eq:LRID}) we get  
\begin{gather}
e^{\beta-b} = e^{\beta'-b'} + e^{b'}  
\label{eq:firstbad}
\end{gather}
and from (\ref{eq:diffLRID}) we get
\begin{align*}
&-a'e^{-a'x-\beta-b}-a e^{-ax-b}-\alpha e^{-(\alpha x+\beta)}\\
+&a'(e^{-a'x+\beta'-b'}+e^{a'x-b'})=0
\end{align*}
Now setting $x=0$ gives
\[
-a'e^{\beta-b}-a e^{b}-\alpha e^{\beta}+a'(e^{\beta'-b'}+e^{b'})=0
\]
and due to (\ref{eq:firstbad}) we get
\[
-ae^{b}-\alpha e^{\beta}=0
\]
which leads to a contradiction. So we do have $\alpha'+a'> \max\{\alpha', a'\}$. This implies $b+\beta=b'+\beta'$; this follows by dividing both sides of (\ref{eq:LRID}) by $e^{-(a+\alpha)x}$ and taking the limit $x\to+\infty$. Therefore we get
\[
e^{-ax-b} + e^{-\alpha x-\beta} = e^{-a'x-b'} + e^{-\alpha'x-\beta'} 
\]
now if $\alpha\neq a$, the dominating term on both sides should be equal with the similar reasoning we did for (\ref{eq:LRID}). As a result $\alpha=\alpha'$ and therefore $a=a'$ (or $\alpha=a'$ and therefore $a=\alpha'$). In either case similar to the proof for \ref{eq:LRID} it follows that $b=b'$ and therefore $\beta=\beta'$ (or $\beta=b'$ and therefore $b=\beta'$). This completes the proof for this case.
\item $a>0$ and $\alpha=0$:
Similar to what has previously shown, it follows that $a'\geq 0$ and $\alpha'\geq 0$. Now we will prove that it cannot be the case that $a'>0$ and $\alpha'>0$, as we argued in case (i) that this is impossible. So $a'=0$ or $\alpha'=0$. In either case it follows that $\alpha'=a$ and $a'=a$ respectively. It follows immediately that $b=b'$ and therefore $\beta=\beta'$ (or $\beta=b'$ and therefore $b=\beta'$).
\item $a=\alpha=0$:
Notice that in this case it is obvious that r.h.s. of (\ref{eq:LRID}) need also to be independent of $x$ which implies $a'=\alpha'=0$. But note that for any non-zero element of either of  $\theta_t$ or $\theta_s$ one can conclude that $b=b'$ and therefore $\beta=\beta'$ (or $\beta=b'$ and therefore $b=\beta'$). Unless all the elements of $\theta_t$ and $\theta_s$ are zero, in which case the identifiability follows. 
\item $a>0$ and $\alpha<0$: Multiply both sides of (\ref{eq:LRID}) with $e^{\alpha|x|}$ and we get
\begin{align}
e^{-(a+(\alpha+|\alpha|))x-\beta-b} &+ e^{-(a+|\alpha|)x-b}\notag\\ 
&+ e^{-(\alpha+|\alpha|)x-\beta}\notag\\
&= e^{-(a'+(\alpha'+|\alpha|))x-\beta'-b'} \notag\\
&+ e^{-(a'+|\alpha|)x-b'} \notag\\ &+ e^{-(\alpha'+|\alpha|)x-\beta'} 
\label{eq:newLRID}
\end{align}
but now the claim of the lemma follows by reapplying (i), (ii) and (iii) to (\ref{eq:newLRID}) with $x$ exponents as $\alpha+|\alpha|$, $a+|\alpha|$, $\alpha+|\alpha|$, $\alpha'+|\alpha|$, $a'+|\alpha|$ and $\alpha'+|\alpha|$.
\item
$a<0$ and $\alpha<0$: This case follows by replacing $x$ with $-x$ on both sides of (\ref{eq:LRID}). This completes the proof of the lemma.
\end{enumerate}
\end{proof}
\subsubsection[Proof of Identifiability of PsychM model]{Proof of Theorem~\ref{thm:PsychM_identifiability} (See page 5\iffalse ~\pageref{thm:PsychM_identifiability}\fi)} 
In this section we present the proof for Theorem \ref{thm:PsychM_identifiability}. The proof is inspired by the proof of Theorem 1 in \citep{ma2018semiparametric}. 
We first address a special case of the problem where $\alpha$ and $a$ are scalars and either of them are zero. 
\begin{lemma}
\label{lem:one_dim_identifiability}
Consider $\mathbf{\Psi}$ when $\alpha\in\mathbb{R}$ and additionally assume $a=0$. Then PsychM model is identifiable up to a sign conversion, i.e. for any set of parameters $(\alpha, \beta,\gamma,\lambda,a,b)$ and $(\alpha', \beta',\gamma',\lambda',a',b')$ if it is the case that 
\begin{align}
&\Psi(x;\alpha, \beta,\gamma,\lambda)\sigma(x;a, b)\notag\\ 
=& \Psi(x;\alpha', \beta',\gamma',\lambda')\sigma(x;a', b')
\label{eq:identifiability_on_dim}
\end{align}
then 
$(a, b,\alpha, \beta, \gamma,\lambda)=(a',b',\alpha', \beta', \gamma',\lambda')$ 
\end{lemma}
\begin{proof}
First assume $a=0$. Then taking the logarithm of both sides of (\ref{eq:identifiability}) we get:
\begin{align}
&\log\Psi(x;\alpha, \beta,\gamma,\lambda)+\log\sigma(x;a, b) \notag\\ 
=&\log\Psi(x;\alpha', \beta',\gamma',\lambda')+\log\sigma(x;a', b')
\label{eq:log_one_dim_identifiability}
\end{align}
Now if $a'<0$ take the limit of $x$ to $+\infty$ on both sides of (\ref{eq:log_one_dim_identifiability}) and we get
\begin{align*}
&\lim_{x\to\infty}\log\Psi(x;\alpha, \beta,\gamma,\lambda)+\log\sigma(x;a, b)
\\
=& C  + \lim_{x\to\infty} \log\sigma(x; a', b')
\end{align*}
where $C$ is a constant. The second summand of the r.h.s. will diverge to $-\infty$ since $a'<0$, which is a contradiction; this is because both summands on the left converge to constant values due to our assumptions; more precisely $\lim_{x\to\infty}\log\Psi(x;\alpha, \beta,\gamma,\lambda)$ is $\log(\gamma)$ or $\log(1-\lambda)$ depending on the sign of $\alpha$, but both of these terms are non-zero since we assumed $\gamma > 0$ and $\lambda<1$.
The similar thing is true for $a'>0$ (when we take $\lim_{x\to-\infty}$). Therefore $a'=0$. 

Now we take the derivative with respect to $x$ from both sides of ($\ref{eq:log_identifiability}$) which gives us
\begin{align*}
&\alpha\frac{\sigma(x;\alpha,\beta)(1-\sigma(x;\alpha,\beta))}{\gamma+(1-\gamma-\lambda)\sigma(x;\alpha, \beta)}\\
=&\alpha'\frac{\sigma(x;\alpha',\beta')(1-\sigma(x;\alpha',\beta'))}{\gamma'+(1-\gamma'-\lambda')\sigma(x;\alpha', \beta')}
\end{align*}
Notice that if $\alpha'=0$ it follows that $\alpha=0$ which is not possible, since we initially assumed $\alpha\neq 0$. So either it is $\alpha \alpha'> 0$ or $\alpha \alpha'< 0$. First assume $\alpha \alpha'> 0$ and WLOG we assume $\alpha>0$. Dividing the l.h.s. with r.h.s. we get
\begin{align}
&\frac{\alpha(\gamma'+(1-\gamma'-\lambda')\sigma(x;\alpha', \beta'))}{\alpha'
(\gamma+(1-\gamma-\lambda)\sigma(x;\alpha, \beta))\sigma(x;\alpha',\beta')}\notag\\
\times&\notag\frac{\sigma(x;\alpha,\beta)(1-\sigma(x;\alpha,\beta))}{(1-\sigma(x;\alpha',\beta'))}=1
\label{eq:division}
\end{align}
And therefore
\begin{align*}
\lim_{x\to\infty}\frac{\alpha}{\alpha'}
&\times\frac{(\gamma'+(1-\gamma'-\lambda')\sigma(x;\alpha', \beta'))}{
(\gamma+(1-\gamma-\lambda)\sigma(x;\alpha, \beta))\sigma(x;\alpha',\beta')}\\
&\times \frac{\sigma(x;\alpha,\beta)(1-\sigma(x;\alpha,\beta))}{(1-\sigma(x;\alpha',\beta'))}=1
\end{align*}
from which it follows that
\begin{gather*}
\lim_{x\to\infty}\frac{\alpha(1-\lambda')}{\alpha'
(1-\lambda))}\times \frac{(1-\sigma(x;\alpha,\beta))}{(1-\sigma(x;\alpha',\beta'))}=1,
\end{gather*}
and hence
\begin{align}
\lim_{x\to\infty}\frac{\alpha(1-\lambda')}{\alpha'
(1-\lambda))}&\times\frac{\exp(-(\alpha^Tx+\beta))}{\exp(-(\alpha'x+\beta'))}
\\
&\times\frac{(1+\exp(-({\alpha'}^Tx+\beta'))}{(1+\exp(-(\alpha^Tx +\beta)))}
=1.
\end{align}
This will lead to
\begin{align}
\lim_{x\to\infty}\frac{\alpha(1-\lambda')}{\alpha'
(1-\lambda)}& \frac{\exp(-(\alpha x+\beta))}{\exp(-(\alpha'x+\beta'))}\notag\\ 
&=\lim_{x\to\infty}\frac{\alpha(1-\lambda')}{\alpha'
(1-\lambda)}\notag\\
&\times\exp(-(\alpha-\alpha')x-(\beta-\beta'))=1
\label{eq:equality_in_powers}
\end{align}
from which we conclude $\alpha=\alpha'$ since if $\alpha\neq\alpha'$ then the exponential term in (\ref{eq:equality_in_powers}) will diverge or converge to zero either of which is impossible. Using (\ref{eq:equality_in_powers}) again, after setting $\alpha=\alpha'$ we get 
\begin{gather}
\frac{\exp(-\beta)(1+\exp(-\beta'))}{\exp(-\beta')(1+\exp(-\beta))}=\frac{1-\lambda}{1-\lambda'}
\label{eq:limits_1}
\end{gather}
and finally taking the limit $x\to-\infty$ in (\ref{eq:division}) we get
\begin{gather}
\frac{\gamma}{\gamma'} = \frac{1+\exp(-\beta')}{1+\exp(-\beta)}
\label{eq:gamma_beta_ratio}
\end{gather}
This time, taking the limit $x\to\infty$ and $x\to-\infty$ in (\ref{eq:log_one_dim_identifiability}) we also get
\begin{gather}
\frac{\gamma}{\gamma'}=\frac{1+\exp(-b)}{1+\exp(-b')}{\ \ \rm and}\ \ \frac{1-\lambda}{1-\lambda'}=\frac{1+\exp(-b)}{1+\exp(-b')}\label{eq:gamma_b_ratio}
\end{gather}
Now from (\ref{eq:limits_1}) and  (\ref{eq:gamma_beta_ratio}) and (\ref{eq:gamma_b_ratio}) it follows that:
$\beta=\beta'$ which implies $\gamma=\gamma'$, $\lambda=\lambda'$ and $b=b'$. Very similar argumentation for $\alpha \alpha'< 0$ will conclude
\[
\alpha=-\alpha',\ \beta=-\beta',\ \gamma=1-\lambda',\  \lambda=1-\gamma'\ {\rm and}\ b=b'
\]
which is a contradiction since $\gamma+\lambda<1$ and $\gamma'+\lambda'<1$, which completes the proof for case $a=0$.
\end{proof}
\begin{lemma}
Suppose $f:\mathbb{R}^n\to\mathbb{R}$ where $f_{\alpha, \beta}(x):=a^Tx+b$ is given where $a\in\mathbb{R}^n$ and $b\in\mathbb{R}$. If there exists an open set $V\in\mathbb{R}^n$ s.t. for all $v\in V$, $f_{\alpha, \beta}(x)=0$ then $f\equiv 0$.
\label{lem:open_mapping}
\end{lemma}
\begin{proof}
Note that $a=0$ then $b=0$ for any $v\in V$ and therefore the proof is complete. Otherwise $f$ is surjective and using open mapping theorem it follows that $f(V)$ is an open set. But $f(V)=\{0\}$ which is a closed set. Since $f(V)$ is clopen which means it can only be $\mathbb{R}$ or $\emptyset$ both leading to contradiction. So in fact $f(\mathbb{R}^n)=\{0\}$ which completes the proof. Also note that this implies $a=0$ and $b=0$.
\end{proof}
Finally a lemma to show that (\ref{eq:identifiability}) will imply $a=a'$.
\begin{lemma}
For any set of parameters $(\alpha, \beta,\gamma,\lambda,a,b)$ and $(\alpha', \beta',\gamma',\lambda',a',b')$ if it is the case that 
\begin{gather}
\Psi(x;\alpha, \beta,\gamma,\lambda)\sigma(x;a, b) =\notag\\ \Psi(x;\alpha', \beta',\gamma',\lambda')\sigma(x;a', b')
\label{eq:identifiability}
\end{gather}
then $a=a'$.
\begin{proof}
The idea of the proof is implicitly used in Lemma \ref{lem:one_dim_identifiability}. Taking the logarithm of both sides we get:
\begin{gather}
\log\Psi(x;\alpha, \beta,\gamma,\lambda)+\log\sigma(x;a, b) =\\
\log\Psi(x;\alpha', \beta',\gamma',\lambda')+\log\sigma(x;a', b')\notag.
\label{eq:log_identifiability}
\end{gather}
We can consider the above equation for any dimension/coordinate $i$ separately by setting  $x_j=0$ for $j\neq i$. So WLOG assume $x\in\mathbb{R}$. Now if $aa'<0$ taking the limit $x\to\infty$ will lead to contradiction similar to Lemma \ref{lem:one_dim_identifiability}. Similarly if $a=0$ or $a'=0$ (exclusive) taking one of the limits $x\to\infty$ or $x\to-\infty$ will lead to contradiction. As such $aa'>0$. WLOG assume $a>0$.
Now divide both sides of (\ref{eq:log_identifiability}) with $x$ taking the limit of $x\to\infty$ we have
\[
\lim_{x\to\infty}\frac{\log \sigma(x;a, b)}{x}\overset{*}{=}\lim_{x\to\infty}a(1-\sigma(x;a,b))=a
\]
where $(*)$ follows from L'Hospitale's rule. This implies $a=a'$. As we can carry this for any dimension it follows that for vector $a$ in  (\ref{eq:identifiability}), we have $a=a'$.
\end{proof}
\label{lem:aa_prime}
\end{lemma}
Finally the following is Lemma 1 from the Appendix of \citep{ma2018semiparametric} which will be used in Theorem \ref{thm:PsychM_identifiability}.
\begin{lemma}
For any nonzero real number $\lambda$,
\[
\lim_{T\to\infty} \frac{1}{2T}
\int_{-T}^{T}e^{is\lambda}ds = 0.
\]
\label{lem:fourier_average}
\end{lemma}
\psychmidentifiability*
\begin{proof}
{
Taking the logarithm of both sides we get:
\begin{align}
&\log\Psi(x;\alpha, \beta,\gamma,\lambda)+\log\sigma(x;a, b)\notag\\ 
=&\log\Psi(x;\alpha', \beta',\gamma',\lambda')+\log\sigma(x;a', b')\notag.
\end{align}
Choose $i$ such that $\alpha_i\neq 0$. Such an $i$ exist since $\alpha\neq 0$. First note that if $a_i=0$ then we can proceed as follows. By setting $x_j=0$ for all $j\neq i$, we can get (\ref{eq:log_identifiability}) for one dimension. Then using Lemma \ref{lem:one_dim_identifiability} we get 
\[
(a_i, b, \gamma,\lambda, \alpha_i,\beta)=(a_i', b', \gamma',\lambda', \alpha_i',\beta')
\]

Notice that at this point for any $j\neq i$ 
we can also right a one-dimensional version of (\ref{eq:log_identifiability}), i.e.
\begin{align*}
&\log\Psi(x_j;\alpha_j, \beta,\gamma,\lambda)+\log\sigma(x_j;a_j, b) \\
=&\log\Psi(x_j;\alpha'_j, \beta',\gamma',\lambda')+\log\sigma(x_j;a'_j, b')\notag.
\end{align*}
And now similar to what we argued in Lemma \ref{lem:one_dim_identifiability} it cannot be the case that $a_j a_j'<0$, and if $a_j a_j'=0$ we can carry the similar argument we just did for index $i$. Now if $a_j a_j'>0$ then we can take $x\to\infty$ or $x\to-\infty$ depending on $a_j<0$ or $a_j>0$ respectively, which would lead to $a_j=a_j'$ with a similar line of argumentation used in (\ref{eq:equality_in_powers}), i.e. using the leading power in the exponential functions in a given ratio. That will consecutively lead to $\alpha_j=\alpha_j'$. Since $j$ was an arbitrary coordinate this will complete the proof for this case. So assume there is no $a_i$ s.t. $a_i=0$. Now choose $i$ such that $\alpha_i\neq 0$. Notice that such $i$ exist since otherwise $\alpha=0$ which is against the assumptions of the theorem. WLOG assume $i=1$. Now assume $\sigma(x)$ and $\Psi(x)$ (where the parameters are dropped) are defined as follows:
\begin{gather}
\Psi(x, \gamma,\lambda):=\gamma+(1-\gamma-\lambda)\sigma(x)\label{eq:psych_short}
\end{gather}
we will use these expressions below to prove the identifibaility. Notice that at the core of all four functions $\Psi(x;\alpha, \beta,\gamma,\lambda), \Psi(x;\alpha', \beta',\gamma',\lambda')$, $\sigma(x;a, b)$ and $\sigma(x;a', b')$ are $\sigma(x)$ and $\Psi(x)$. We will attempt to write the Fourier transform of log of these functions, in a more canonical form based on the Fourier transform of $\sigma(x)$ and $\Psi(x)$ to tackle the identifiability problem. Before doing so notice that the parameters $(a, b, \alpha, \beta)$ and respectively $(a'
, b', \alpha', \beta')$ are linearly related to the features. We will use this fact to rewrite this linear relationship for the application of Fourier transform in one dimension. To elaborate on this let's consider 
\[
\sigma(x;a,b)=(1+\exp(-(a^Tx + b)))^{-1}.
\]
this can be rewritten as
\[
\sigma(x;a,b)=\sigma(a^T x + b) = \sigma(a_{\bar{1}}^Tx_{\bar{1}} + a_1x_1 + b).
\]
where $v_{\bar{i}}\in\mathbb{R}^n$ is defined as the vector that is derived from vector $v\in\mathbb{R}^n$ by removing its $i$-th coordinate. A similar way of representation can be used for all the other three functions in (\ref{eq:log_identifiability}). We will use this representation below to apply one dimensional Fourier transform on both sides of (\ref{eq:log_identifiability}).

For a function $f(x)$, define the Fourier transform ($\mathcal{F}$(.)) of it, $\hat{f}$ as 
\[
\hat{f}(\nu):=\int_{-\infty}^{\infty}f_{\alpha, \beta}(x)e^{-2\pi i\nu x}dx
\]
Before applying the Fourier transform it is important to ensure the existence of such transformation. In fact terms in (\ref{eq:log_identifiability}) do not fall into such category as they are unbounded with unbounded support. For this reason we will take second derivative of (\ref{eq:log_identifiability}) w.r.t. $x_1$, which leads to:
{\tiny
\begin{align}
&\left[\frac{\alpha_1^2 \sigma_{\alpha,\beta}(1-\sigma_{\alpha,\beta})}{(\gamma+(1-\gamma-\lambda)\sigma_{\alpha,\beta})^2}\times\notag \right.
\left.\vphantom{\frac{1}{2}}\left((1-2\sigma_{\alpha,\beta})\Psi_{\alpha,\beta}-\sigma_{\alpha,\beta}(1-\sigma_{\alpha,\beta})\right)\right]\notag\\
- & a_1^2(\sigma_{a, b})(1-\sigma_{a,b})=
\left[\frac{{\alpha'_1}^2 \sigma_{{\alpha'},{\beta'}}(1-\sigma_{{\alpha'},{\beta'}})}{(\gamma_1'+(1-\gamma'-\lambda')\sigma_{{\alpha'},{\beta'}})^2} \right.\notag\\
\times &\left.\vphantom{\frac{1}{2}}\left((1-2\sigma_{{\alpha'},{\beta'}})\Psi_{{\alpha'},{\beta'}}-\sigma_{{\alpha'},{\beta'}}(1-\sigma_{{\alpha'},{\beta'}})\right)\right]-\\
&{a'}^2(\sigma_{a', b'})(1-\sigma_{a',b'})
\label{eq:second_derivative}
\end{align}}
where $\sigma_{\alpha,\beta}$ and $\Psi_{\alpha, \beta}$ are synonyms to $\sigma(x;\alpha,\beta)$ and $\Psi(x;\alpha, \beta, \gamma, \lambda)$ respectively and are used as short-hands. Define
{\tiny
\begin{flalign*}
    &f_{\gamma, \lambda}(\alpha^T x+\beta):=f(\alpha^T x+\beta):=f_{\alpha, \beta}(x)&\\
    :=&f_{\alpha, \beta,\gamma, \lambda}(x):=f_{\gamma, \lambda}(x)&\\
    :=&\left[\frac{\alpha_1^2 \sigma_{\alpha,\beta}(1-\sigma_{\alpha,\beta})}{(\gamma+(1-\gamma-\lambda)\sigma_{\alpha,\beta})^2}\notag \left((1-2\sigma_{\alpha,\beta})\Psi_{\alpha,\beta}-\sigma_{\alpha,\beta}(1-\sigma_{\alpha,\beta})\right)\right]
\end{flalign*}
}
and $g(a^T x+b):=g_{a, b}(x):=a_1^2(\sigma_{a, b})(1-\sigma_{a,b})$. 
Now notice that we have
\begin{align*}
\int_{-\infty}^{\infty}\left|f_{\alpha, \beta}(x)\right|dx &\leq 
\int_{-\infty}^{\infty}2\frac{\alpha_1^2 \sigma_{\alpha,\beta}(1-\sigma_{\alpha,\beta})}{(\gamma+(1-\gamma-\lambda)\sigma_{\alpha,\beta})^2}dx\\ &\leq \int_{-\infty}^{\infty}2\frac{\alpha_1^2 \sigma_{\alpha,\beta}(1-\sigma_{\alpha,\beta})}{\gamma^2}dx\\
&=2\frac{\alpha_1^2}{\gamma^2}\int_{-\infty}^{\infty} \sigma_{\alpha,\beta}(1-\sigma_{\alpha,\beta})dx\\
&=4\frac{\alpha_1^2}{|\alpha_1|\gamma^2}
\end{align*}
which means $f_{\alpha, \beta}(x)\in L_1$. Similarly
\[
\int_{-\infty}^{\infty}|g_{a, b}(x)|=\frac{a_1^2}{|a_1|}
\]
and therefore $g_{a, b}(x)\in L_1$. As such $f_{\alpha, \beta}(x)$ and $g_{a, b}(x)$ have Fourier transforms. Now we take the Fourier transform of both sides of (\ref{eq:second_derivative}) with respect to $x_1$ getting:
\begin{gather}
\mathcal{F}\left(f_{\alpha,\beta}(x)\right) + \mathcal{F}\left(g_{a, b}(x)\right) = \mathcal{F}\left(f_{\alpha',\beta'}(x)\right) + \mathcal{F}\left(g_{a', b'}(x)\right)\label{eq:left_fourier}
\end{gather}
Consider the first term on the l.h.s. of (\ref{eq:left_fourier}):
\begin{align*}
&\mathcal{F}\left(f_{\alpha, \beta}(x)\right) = \mathcal{F}\left(f(\alpha^T x+\beta)\right) \mathcal{F}\left(f(\alpha_{\bar{1}}^Tx_{\bar{1}} + \alpha_1x_1 + \beta)\right)\\
=&\int_{-\infty}^\infty e^{-2\pi i \nu x_1} \left(f(\alpha_{\bar{1}}^Tx_{\bar{1}} + \alpha_1x_1 + \beta)\right) dx_1\\
=&\int_{-\infty}^\infty e^{-2\pi i \frac{\nu}{\alpha_1} \alpha_1x_1} \left(f(\alpha_{\bar{1}}^Tx_{\bar{1}} + \alpha_1x_1 + \beta)\right) dx_1\\
=&\frac{e^{2\pi i \frac{\nu}{\alpha_1} (\alpha_{\bar{1}}^Tx_{\bar{1}}+\beta)}}{|\alpha_1|}\reallywidehat{f_{\gamma, \lambda}}\left(\frac{\nu}{\alpha_1}\right)
\end{align*}
Applying the similar modification to the other 3 terms in (\ref{eq:left_fourier}) we get:
{\tiny
\begin{gather*}
\frac{e^{2\pi i \frac{\nu}{\alpha_1} (\alpha_{\bar{1}}^Tx_{\bar{1}}+\beta)}}{|\alpha_1|}\reallywidehat{f_{\gamma, \lambda}}\left(\frac{\nu}{\alpha_1}\right) +
\frac{e^{2\pi i \frac{\nu}{a_1} (a_{\bar{1}}^Tx_{\bar{1}}+b)}}{|a_1|}\reallywidehat{g}\left(\frac{\nu}{a_1}\right) = \\
\frac{e^{2\pi i \frac{\nu}{\alpha'_1} ({\alpha_{\bar{1}}'}^Tx_{\bar{1}}+\beta')}}{|\alpha'_1|}\reallywidehat{f_{\gamma', \lambda'}}\left(\frac{\nu}{\alpha'_1}\right) + 
\frac{e^{2\pi i \frac{\nu}{a'_1} ({a_{\bar{1}}'}^Tx_{\bar{1}}+b')}}{|a'_1|}\reallywidehat{g}\left(\frac{\nu}{a'_1}\right)
\end{gather*}
}
Consider the following cases: 
\begin{enumerate*}[label=(\roman*)]
\item 
$\frac{\alpha_{\bar{1}}}{\alpha_1}\neq \frac{a_{\bar{1}}}{a_1}$     
\label{eq:ineq_1}
\item 
$\frac{\alpha_{\bar{1}}}{\alpha_1}\neq \frac{\alpha'_{\bar{1}}}{\alpha'_1}$
\label{eq:ineq_2}
\item 
$\frac{\alpha_{\bar{1}}}{\alpha_1}\neq \frac{a'_{\bar{1}}}{a'_1}$ 
\label{eq:ineq_3}
\end{enumerate*}. First we assume \ref{eq:ineq_1}, \ref{eq:ineq_2}, and \ref{eq:ineq_3} hold. Now consider the following hyperplanes: 

\begin{gather*}
H_{a}=\{x\in\mathbb{R}^n|\left(\frac{\alpha_{\bar{1}}}{\alpha_1}-\frac{a_{\bar{1}}^T}{a_1}\right)^Tx+\frac{\beta}{\alpha_1}-\frac{b}{a_1}= 0\}\\
H_{\alpha'}=\{x\in\mathbb{R}^n|\left(\frac{\alpha_{\bar{1}}}{\alpha_1}-\frac{\alpha'_{\bar{1}}}{\alpha'_1}\right)^Tx+\frac{\beta}{\alpha_1}-\frac{\beta'}{\alpha'_1}=0\}\\
H_{a'}=\{x\in\mathbb{R}^n|\left(\frac{\alpha_{\bar{1}}}{\alpha_1}-\frac{a'_{\bar{1}}}{a'_1}\right)^Tx+\frac{\beta}{\alpha_1}-\frac{b'}{a'_1}=0\}
\end{gather*}
Now notice that $H:=H_a\cup H_\alpha' \cup H_{a'}$ is a closed set, therefore $\mathbb{R}^n\setminus H$ is open. Take an arbitrary open disk $V$ in $\mathbb{R}^n\setminus H$. Based on the definition of $V$  it follows that for any $w\in V$
\begin{gather*}
\left(\frac{\alpha_{\bar{1}}}{\alpha_1}-\frac{a_{\bar{1}}}{a_1}\right)^Tw+\frac{\beta}{\alpha_1}-\frac{b}{a_1}\neq 0\\
\left(\frac{\alpha_{\bar{1}}}{\alpha_1}-\frac{\alpha'_{\bar{1}}}{\alpha'_1}\right)^Tw+\frac{\beta}{\alpha_1}-\frac{\beta'}{\alpha'_1}\neq 0\\
\left(\frac{\alpha_{\bar{1}}}{\alpha_1}-\frac{a'_{\bar{1}}}{a'_1}\right)^Tw+\frac{\beta}{\alpha_1}-\frac{b'}{a'_1}\neq 0
\end{gather*}
Take such a $w$. Then for any $s\in\mathbb{R}$ except the solution of $\left(s\left(\frac{\alpha_{\bar{1}}}{\alpha_1}-\frac{a_{\bar{1}}}{a_1}\right)^Tw + \frac{\beta}{\alpha_1} - \frac{b}{a_1}\right)=0$, we have $\left(s\left(\frac{\alpha_{\bar{1}}}{\alpha_1}-\frac{a_{\bar{1}}}{a_1}\right)^Tw + \frac{\beta}{\alpha_1} - \frac{b}{a_1}\right)\neq 0$. Dividing both sides with $\exp\left(2\pi i\nu\left(\frac{\alpha_{\bar{1}}}{\alpha_1}^Tw + \frac{\beta}{\alpha_1}\right)\right)$ we get:
\begin{gather*}
\frac{\reallywidehat{f_{\gamma,\lambda}}\left(\frac{\nu}{\alpha_1}\right)}{|\alpha_1|} + 
\frac{e^{2\pi i \nu\left(s(\frac{a_{\bar{1}}}{a_1}-\frac{\alpha_{\bar{1}}}{\alpha_1})^Tw+\frac{\beta}{\alpha_1}-\frac{b}{a_1}\right)}}{|a_1|}\reallywidehat{g}\left(\frac{\nu}{a_1}\right) = \notag\\
\frac{e^{2\pi i \nu\left(s(\frac{\alpha'_{\bar{1}}}{\alpha'_1}-\frac{\alpha_{\bar{1}}}{\alpha_1})^Tw+\frac{\beta}{\alpha_1}-\frac{\beta'}{\alpha'_1}\right)}}{|\alpha'_1|}\reallywidehat{f{\gamma',\lambda'}}\left(\frac{\nu}{\alpha'_1}\right) + \notag\\
\frac{e^{2\pi i \nu\left(s(\frac{a'_{\bar{1}}}{a'_1}-\frac{\alpha_{\bar{1}}}{\alpha_1})^Tw+\frac{\beta}{\alpha_1}-\frac{b'}{a'_1}\right)}}{|a'_1|}\reallywidehat{g}\left(\frac{\nu}{a'_1}\right)
\label{eq:division_by_exp}
\end{gather*}
Applying Lemma \ref{lem:fourier_average}, i.e. dividing both sides with $2T$ and integrating with $\int_{-T}^{T}(.)ds$, and finally taking the limit $T\to\infty$ we get
\[
\frac{\reallywidehat{f_{\gamma,\lambda}}\left(\frac{\nu}{\alpha_1}\right)}{|\alpha_1|}=0
\]
for all $\nu$'s, which is obviously impossible. As such it follows that 
\begin{gather*}
\left(\frac{\alpha_{\bar{1}}}{\alpha_1}-\frac{\alpha'_{\bar{1}}}{\alpha'_1}\right)^Tw+\frac{\beta}{\alpha_1}-\frac{\beta'}{\alpha'_1}=0\\
{\rm or}\\
\left(\frac{\alpha_{\bar{1}}}{\alpha_1}-\frac{a'_{\bar{1}}}{a'_1}\right)^Tw+\frac{\beta}{\alpha_1}-\frac{b'}{a'_1}=0
\end{gather*}
which is not an exclusive or. Note that since this holds for any $w\in V$ using Lemma \ref{lem:open_mapping} it follows that 
$\alpha_1'\alpha_{\bar{1}}=\alpha_1\alpha'_{\bar{1}}$ or $a'_1\alpha_{\bar{1}}=\alpha_1a'_{\bar{1}}$ which are equivalent to $\alpha=\frac{\alpha_1}{\alpha_1'} \alpha'$ or $\alpha=\frac{\alpha_1}{a_1'} a'$ respectively. This is in contradiction with either of \ref{eq:ineq_1}, \ref{eq:ineq_2}, or \ref{eq:ineq_3}.   Notice that we cannot have the negation of \ref{eq:ineq_3} or \ref{eq:ineq_1} exclusively because of Lemma \ref{lem:aa_prime}. So let's assume \ref{eq:ineq_1} and \ref{eq:ineq_3} but the negation of \ref{eq:ineq_2} holds. Applying Lemma \ref{lem:fourier_average}, i.e. dividing both sides with $2T$ and integrating with $\int_{-T}^{T}(.)ds$, and finally taking the limit $T\to\infty$ we get
\[
\frac{\reallywidehat{f_{\gamma,\lambda}}\left(\frac{\nu}{\alpha_1}\right)}{|\alpha_1|} + \frac{\reallywidehat{g}\left(\frac{\nu}{a_1}\right)}{|a_1|} = 
\frac{\reallywidehat{g}\left(\frac{\nu}{a'_1}\right)}{|a'_1|}
\]
but notice that due to Lemma \ref{lem:aa_prime} $a_1=a_1'$ and such we will have 
\[
\frac{\reallywidehat{f_{\gamma,\lambda}}\left(\frac{\nu}{\alpha_1}\right)}{|\alpha_1|}=0
\]
which again is impossible. Also if neither of \ref{eq:ineq_1}, \ref{eq:ineq_2}, or \ref{eq:ineq_3} hold then using Lemmas \ref{lem:aa_prime} and \ref{lem:open_mapping} we will conclude that $a=a'=\alpha=\alpha'$ and $b=b'=\beta=\beta'$ which completes the proof. Therefore the only possible remaining case is where \ref{eq:ineq_2} does not hold but \ref{eq:ineq_1} and \ref{eq:ineq_3} do hold. Applying Lemma \ref{lem:fourier_average}, i.e. dividing both sides with $2T$ and integrating with $\int_{-T}^{T}(.)ds$, and finally taking the limit $T\to\infty$ we get
\[
\left(\frac{\alpha_{\bar{1}}}{\alpha_1}-\frac{\alpha'_{\bar{1}}}{\alpha'_1}\right)^Tw+\frac{\beta}{\alpha_1}-\frac{\beta'}{\alpha'_1}=0
\]
which using Lemma \ref{lem:open_mapping} implies
$\alpha=\frac{\alpha_1}{\alpha'_1}\alpha'$ and $\beta=\frac{\alpha_1}{\alpha_1'}\beta'$. Additionally we will have
\begin{gather*}
\frac{e^{2\pi i \nu\left(s(\frac{a_{\bar{1}}}{a_1}-\frac{\alpha_{\bar{1}}}{\alpha_1})^Tw+\frac{\beta}{\alpha_1}-\frac{b}{a_1}\right)}}{|a_1|}\reallywidehat{g}\left(\frac{\nu}{a_1}\right) = \notag\\
\frac{e^{2\pi i \nu\left(s(\frac{a'_{\bar{1}}}{a'_1}-\frac{\alpha_{\bar{1}}}{\alpha_1})^Tw+\frac{\beta}{\alpha_1}-\frac{b'}{a'_1}\right)}}{|a'_1|}\reallywidehat{g}\left(\frac{\nu}{a'_1}\right)
\end{gather*}
but $a_1=a_1'$ due to Lemma \ref{lem:aa_prime}. This will imply
\[
e^{2\pi i \nu\left(s(\frac{a_{\bar{1}}}{a_1}-\frac{\alpha_{\bar{1}}}{\alpha_1})^Tw+\frac{\beta}{\alpha_1}-\frac{b}{a_1}\right)}=e^{2\pi i \nu\left(s(\frac{a_{\bar{1}}}{a'_1}-\frac{\alpha_{\bar{1}}}{\alpha_1})^Tw+\frac{\beta}{\alpha_1}-\frac{b'}{a'_1}\right)}
\]
and as such, using Lemmas \ref{lem:open_mapping} and \ref{lem:aa_prime} will lead to $b=b'$. But this would imply 
\[
\Psi(x;\alpha,\beta, \gamma,\lambda)=\Psi(x;\alpha',\beta', \gamma',\lambda')
\]
Now consider things for $x_1$. Assume $\alpha_1\alpha_1'<0$. If $\alpha_1<0$ we take the limit $x\to\infty$ will lead to  $\gamma=1-\lambda'$ and $x\to-\infty$ will lead to $\gamma'=1-\lambda$. But note that $\gamma+\lambda < 1$ and $\gamma'+\lambda' < 1$ based on our assumptions and this will lead to a contradiction.  Similar is true for $\alpha_1>0$. So $\alpha_1\alpha_1'>0$. This time taking the same limits implies $\gamma=\gamma'$ and $\lambda=\lambda'$, which will also lead to $\alpha_1=\alpha_1'$ and $\beta=\beta'$. But then $\alpha=\frac{\alpha_1}{\alpha_1'}\alpha'$ implies $\alpha=\alpha'$. This completes the proof.
}
\end{proof}
\subsection{Other results about real-world experiments}
As noted in subsection \ref{subsec:real_world} beside classification accuracy as a measure of success of our introduced models, we had consistent results across other classification metrics such as f1, ROCAUC, and Brier score. Below we provide these results in Tables \ref{tab:f1}, \ref{tab:AUC}, and \ref{tab:Brier}:

\begin{table}[H]
\centering
\scalebox{0.8}{  \begin{tabular}{|p{1cm}|p{1.2cm}|p{1.2cm}|p{1.2cm}|p{1.2cm}|p{1.2cm}|l}
    \toprule
    &  fsa & mta & sph & hth & mba\\
    \midrule
    SPM & $\textit{0.89479}$& ${\bf 0.91477}$ &$\textit{0.92218}$& ${\bf 0.92758}$ & $\textit{0.93678}$\\
    \midrule
    PsychM & ${\bf 0.89584}$ & $\textit{0.91418}$ & ${\bf 0.92377}$&  $\textit{0.92644}$ & ${\bf 0.93729}$\\
    \midrule
    Na\"ive  & $0.88341$ & $0.90391$ & $0.91356$ &  $0.91942$ & $0.93038$ \\
    \midrule
    Elkan & $0.88440$ & $0.90373$ &  $0.91308$ & $0.91759$ & $0.92901$  \\
    \midrule
    Real & $0.95993$ & $0.95919$& $0.95962$& $0.96012$ & $0.95991$\\
\bottomrule
\end{tabular}
}
 \caption{f1 score for different models across different datasets, in bootstraps over 200 trials.}
\label{tab:f1}   
\end{table}
\begin{table}[H]
\centering
\scalebox{0.8}{  \begin{tabular}{|p{1cm}|p{1.2cm}|p{1.2cm}|p{1.2cm}|p{1.2cm}|p{1.2cm}|l}
    \toprule
    &  fsa & mta & sph & hth & mba\\
    \midrule
    SPM & $\textit{0.97333}$& ${\bf 0.97891}$ &$\textit{0.98068}$& $\textit{0.98161}$ & $\textit{0.98342}$\\
    \midrule
    PsychM & ${\bf 0.97643}$ & $\textit{0.98076}$ & ${\bf 0.98257}$&  $\textbf{0.98283}$ & ${\bf 0.98495}$\\
    \midrule
    Na\"ive  & $0.97480$ & $0.97901$ & $0.98098$ &  $0.98147$ & $0.98367$ \\
    \midrule
    Elkan & $0.97376$ & $0.97786$ &  $0.97947$ & $0.98015$ & $0.98235$  \\
    \midrule
    Real & $0.99286$ & $0.99317$& $0.99317$& $0.99327$ & $0.99326$\\
\bottomrule
\end{tabular}
}
 \caption{AUC scores for different models across different datasets, in bootstraps over 200 trials.}
\label{tab:AUC}   
\end{table}
\begin{table}[H]
\centering
\scalebox{0.8}{  \begin{tabular}{|p{1cm}|p{1.2cm}|p{1.2cm}|p{1.2cm}|p{1.2cm}|p{1.2cm}|l}
    \toprule
    &  fsa & mta & sph & hth & mba\\
    \midrule
    SPM & $\textit{0.07813}$& $\textit{0.06450}$ &$\textit{0.05940}$& $\textbf{0.05556}$ & $\textit{0.04981}$\\
    \midrule
    PsychM & ${\bf 0.07535}$ & $\textbf{0.06384}$ & ${\bf 0.05797}$&  $\textit{0.05615}$ & ${\bf 0.04873}$\\
    \midrule
    Na\"ive  & $0.08851$ & $0.07363$ & $0.06633$ &  $0.06305$ & $0.05453$ \\
    \midrule
    Elkan & $0.08744$ & $0.07351$ &  $0.06634$ & $0.06421$ & $0.05551$  \\
    \midrule
    Real & $0.03055$ & $0.03054$& $0.03077$& $0.03030$ & $0.03040$\\
\bottomrule
\end{tabular}
}
 \caption{Brier scores for different models across different datasets, in bootstraps over 200 trials. The lower is better.}
\label{tab:Brier}   
\end{table}

You can see that SPM and PsychM consistently outperform all the other LePU methods on all the four scores that we considered here (together with the results reported in Table \ref{tab:real_world}).



In this appendix we prove the following theorem from
Section~6.2:

\vskip 0.2in
\bibliography{main.bib}

\end{document}